\def\year{2016}\relax
\documentclass[letterpaper]{article}
\usepackage{aaai16}
\usepackage{times}
\usepackage{helvet}
\usepackage{courier}
\usepackage{url}
\usepackage[round]{natbib}
\usepackage{amsmath,amsfonts,amssymb,bm,graphicx,amsthm}
\usepackage{subfig}
\usepackage{wrapfig}
\usepackage{dblfloatfix}

\def\reals{\mathbb{R}}

\def\E#1{\mathbb{E}\left[#1\right]}

\def\th{\bm\theta}
\def\g{\gamma}

\DeclareMathOperator*{\argmax}{\arg\!\max}

\newtheorem{theorem}{Theorem}

\begin{document}
\title{Deep Reinforcement Learning with Double Q-learning}

\author{Hado van Hasselt \and
Arthur Guez \and
David Silver \\
Google DeepMind}

\hyphenation{Zaxxon Double}

\maketitle
\begin{abstract}
The popular Q-learning algorithm is known to overestimate action values under certain conditions. It was not previously known whether, in practice, such overestimations are common, whether they harm performance, and whether they can generally be prevented.  In this paper, we answer all these questions affirmatively.  In particular, we first show that the recent DQN algorithm, which combines Q-learning with a deep neural network, suffers from substantial overestimations in some games in the Atari 2600 domain.  We then show that the idea behind the Double Q-learning algorithm, which was introduced in a tabular setting, can be generalized to work with large-scale function approximation.  We propose a specific adaptation to the DQN algorithm and show that the resulting algorithm not only reduces the observed overestimations, as hypothesized, but that this also leads to much better performance on several games.
\end{abstract}

\noindent The goal of reinforcement learning \citep{SuttonBarto:1998} is to learn good policies for sequential decision problems, by optimizing a cumulative future reward signal.
Q-learning \citep{Watkins:1989} is one of the most popular reinforcement learning algorithms, but it is known to sometimes learn unrealistically high action values because it includes a maximization step over estimated action values, which tends to prefer overestimated to underestimated values.

In previous work, overestimations have been attributed to insufficiently flexible function approximation \citep{Thrun:1993} and noise \citep{vanHasselt:2010,vanHasselt:2011}.  In this paper, we unify these views and show overestimations can occur when the action values are inaccurate, irrespective of the source of approximation error.  Of course, imprecise value estimates are the norm during learning, which indicates that overestimations may be much more common than previously appreciated.

It is an open question whether, if the overestimations do occur, this negatively affects performance in practice. Overoptimistic value estimates are not necessarily a problem in and of themselves.  If all values would be uniformly higher then the relative action preferences are preserved and we would not expect the resulting policy to be any worse.  Furthermore, it is known that sometimes it is good to be optimistic: optimism in the face of uncertainty is a well-known exploration technique \citep{Kaelbling:1996}.  If, however, the overestimations are not uniform and not concentrated at states about which we wish to learn more, then they might negatively affect the quality of the resulting policy.  \citet{Thrun:1993} give specific examples in which this leads to suboptimal policies, even asymptotically.

To test whether overestimations occur in practice and at scale, we investigate the performance of the recent DQN algorithm \citep{Mnih:2015}.  DQN combines Q-learning with a flexible deep neural network and was tested on a varied and large set of deterministic Atari 2600 games, reaching human-level performance on many games.  In some ways, this setting is a best-case scenario for Q-learning, because the deep neural network provides flexible function approximation with the potential for a low asymptotic approximation error, and the determinism of the environments prevents the harmful effects of noise.
Perhaps surprisingly, we show that even in this comparatively favorable setting DQN sometimes substantially overestimates the values of the actions.

We show that the idea behind the Double Q-learning algorithm \citep{vanHasselt:2010}, which was first proposed in a tabular setting, can be generalized to work with arbitrary function approximation, including deep neural networks.  We use this to construct a new algorithm we call Double DQN.  We then show that this algorithm not only yields more accurate value estimates, but leads to much higher scores on several games.  This demonstrates that the overestimations of DQN were indeed leading to poorer policies and that it is beneficial to reduce them.  In addition, by improving upon DQN we obtain state-of-the-art results on the Atari domain.

\section{Background}
\label{sec:qoveropt}

To solve sequential decision problems we can learn estimates for the optimal value of each action, defined as the expected sum of future rewards
when taking that action and following the optimal policy thereafter.  Under a given policy $\pi$, the true value of an action $a$ in a state $s$ is
\[
Q_{\pi}(s,a) \equiv \E{ R_1 + \g R_2 + \ldots \mid S_0 = s, A_0 = a, \pi } \,,
\]
where $\g \in [0,1]$ is a discount factor that trades off the importance of immediate and later rewards.
The optimal value is then $Q_*(s,a) = \max_{\pi} Q_{\pi}(s,a)$. An optimal policy is easily derived from the optimal values by selecting the highest-valued action in each state.

Estimates for the optimal action values can be learned using Q-learning \citep{Watkins:1989}, a form of temporal difference learning \citep{Sutton:1988}.
Most interesting problems are too large to learn all action values in all states separately.  Instead, we can learn a parameterized value function $Q(s,a;\th_t)$.
The standard Q-learning update for the parameters after taking action $A_t$ in state $S_t$ and observing the immediate reward $R_{t+1}$ and resulting state $S_{t+1}$ is then
\begin{equation}\label{Q}
\th_{t+1} = \th_t + \alpha (Y^{\text{Q}}_t - Q(S_t,A_t;\th_t)) \nabla_{\th_t} Q(S_t,A_t;\th_t) \,.
\end{equation}
where $\alpha$ is a scalar step size and the target $Y^{\text{Q}}_t$ is defined as
\begin{equation}\label{TDQ}
Y^{\text{Q}}_t \equiv R_{t+1} + \gamma \max_a Q(S_{t+1}, a; \th_t) \,.
\end{equation}
This update resembles stochastic gradient descent, updating the current value $Q(S_t,A_t;\th_t)$ towards a target value $Y^{\text{Q}}_t$.

\subsection{Deep Q Networks} A deep Q network (DQN) is a multi-layered neural network that for a given state $s$ outputs a vector of action values $Q(s,\cdot\,;\th)$, where $\th$ are the parameters of the network.  For an $n$-dimensional state space and an action space containing $m$ actions, the neural network is a function from $\reals^n$ to $\reals^m$.
Two important ingredients of the DQN algorithm as proposed by \citet{Mnih:2015} are the use of a target network, and the use of experience replay.
The target network, with parameters $\th^-$, is the same as the online network except that its parameters are copied every $\tau$ steps from the online network, so that then $\th^-_t = \th_t$, and kept fixed on all other steps.  The target used by DQN is then
\begin{equation}\label{DQN}
Y^{\text{DQN}}_t \equiv R_{t+1} + \gamma \max_a Q(S_{t+1}, a; \th^-_t) \,.
\end{equation}
For the experience replay \citep{Lin:1992}, observed transitions are stored for some time and sampled uniformly from this memory bank to update the network. Both the target network and the experience replay dramatically improve the performance of the algorithm \citep{Mnih:2015}.

\subsection{Double Q-learning} The max operator in standard Q-learning and DQN, in \eqref{TDQ} and \eqref{DQN}, uses the same values both to select and to evaluate an action.  This makes it more likely to select overestimated values, resulting in overoptimistic value estimates.
To prevent this, we can decouple the selection from the evaluation.  This is the idea behind Double Q-learning \citep{vanHasselt:2010}.

In the original Double Q-learning algorithm, two value functions are learned by assigning each experience randomly to update one of the two value functions, such that there are two sets of weights, $\th$ and $\th'$.  For each update, one set of weights is used to determine the greedy policy and the other to determine its value.
For a clear comparison, we can first untangle the selection and evaluation in Q-learning and rewrite its target \eqref{TDQ} as
\[
Y^{\text{Q}}_t = R_{t+1} + \g Q(S_{t+1}, \argmax_a Q(S_{t+1}, a; \th_t); \th_t ) \,.
\]
The Double Q-learning error can then be written as
\begin{equation}\label{TDDQ}
Y^{\text{DoubleQ}}_t \!\equiv R_{t+1} + \g Q(S_{t+1}, \argmax_a Q(S_{t+1}, a; \th_t); \th'_t ) \,.
\end{equation}
Notice that the selection of the action, in the $\argmax$, is still due to the online weights $\th_t$.  This means that, as in Q-learning, we are still estimating the value of the greedy policy according to the current values, as defined by $\th_t$.  However, we use the second set of weights $\th_t'$ to fairly evaluate the value of this policy.  This second set of weights can be updated symmetrically by switching the roles of $\th$ and $\th'$.

\section{Overoptimism due to estimation errors}
\label{sec:learningoverest}

Q-learning's overestimations were first investigated by \citet{Thrun:1993}, who showed that if the action values contain random errors uniformly distributed in an interval $[-\epsilon,\epsilon]$ then each target is overestimated up to $\g \epsilon \frac{m-1}{m+1}$, where $m$ is the number of actions.  In addition, \citeauthor{Thrun:1993} give a concrete example in which these overestimations even asymptotically lead to sub-optimal policies, and show the overestimations manifest themselves in a small toy problem when using function approximation.
Later \citet{vanHasselt:2010} argued that noise in the environment can lead to overestimations even when using tabular representation, and proposed Double Q-learning as a solution.

In this section we demonstrate more generally that estimation errors of any kind can induce an upward bias, regardless of whether these errors are due to environmental noise, function approximation, non-stationarity, or any other source.  This is important, because in practice any method will incur some inaccuracies during learning, simply due to the fact that the true values are initially unknown.

The result by \citet{Thrun:1993} cited above gives an upper bound to the overestimation for a specific setup, but it is also possible, and potentially more interesting, to derive a lower bound.
\begin{theorem}\label{thm:lower_bound}
Consider a state $s$ in which all the true optimal action values are equal at $Q_*(s,a) = V_*(s)$ for some $V_*(s)$. Let $Q_t$ be arbitrary value estimates that are on the whole unbiased in the sense that $\sum_a ( Q_t(s,a) - V_*(s) ) = 0$, but that are not all correct, such that $\frac{1}{m} \sum_a ( Q_t(s,a) - V_*(s) )^2 = C$ for some $C > 0$, where $m \ge 2$ is the number of actions in $s$. Under these conditions, $\max_a Q_t(s,a) \ge V_*(s) + \sqrt{\frac{C}{m-1}}$.  This lower bound is tight. Under the same conditions, the lower bound on the absolute error of the Double Q-learning estimate is zero. (Proof in appendix.)
\end{theorem}
Note that we did not need to assume that estimation errors for different actions are independent.
This theorem shows that even if the value estimates are on average correct, estimation errors of any source can drive the estimates up and away from the true optimal values.

The lower bound in Theorem \ref{thm:lower_bound} decreases with the number of actions. This is an artifact of considering the lower bound, which requires very specific values to be attained. More typically, the overoptimism increases with the number of actions as shown in Figure \ref{Gauss_bars}.
\begin{figure}[t]
\begin{center}
\includegraphics[width=3.3in]{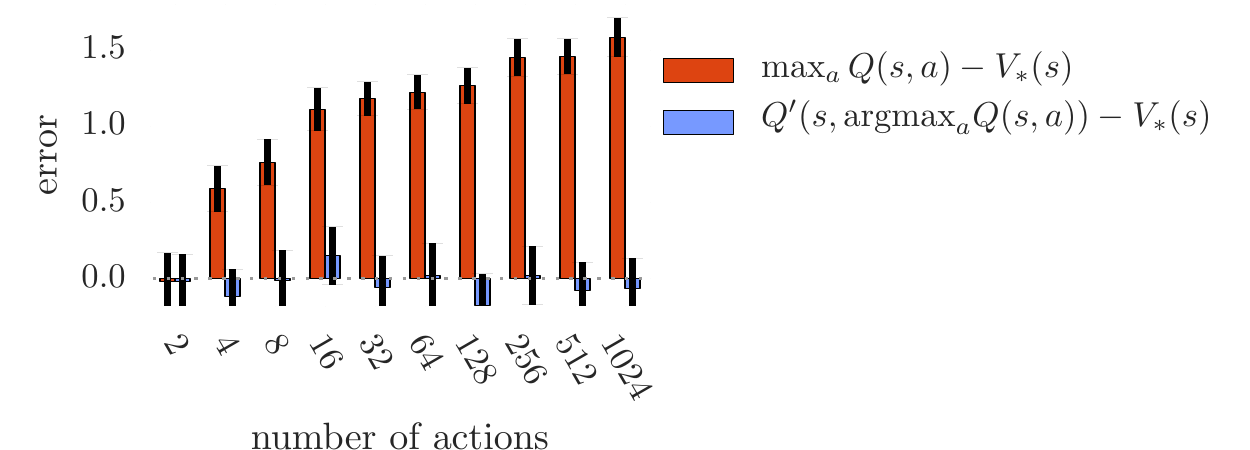}
\caption{\label{Gauss_bars} The orange bars show the bias in a single Q-learning update when the action values are $Q(s,a) = V_*(s) + \epsilon_a$ and the errors $\{\epsilon_a\}_{a=1}^m$ are independent standard normal random variables.  The second set of action values $Q'$, used for the blue bars, was generated identically and independently.  All bars are the average of 100 repetitions.  }
\end{center}
\end{figure}
Q-learning's overestimations there indeed increase with the number of actions, while Double Q-learning is unbiased.
As another example, if for all actions $Q_*(s,a) = V_*(s)$ and the estimation errors $Q_t(s,a) - V_*(s)$ are uniformly random in $[-1,1]$, then the overoptimism is $\frac{m-1}{m+1}$. (Proof in appendix.)
\begin{figure*}[t]
\centering
\includegraphics[width=6.8in]{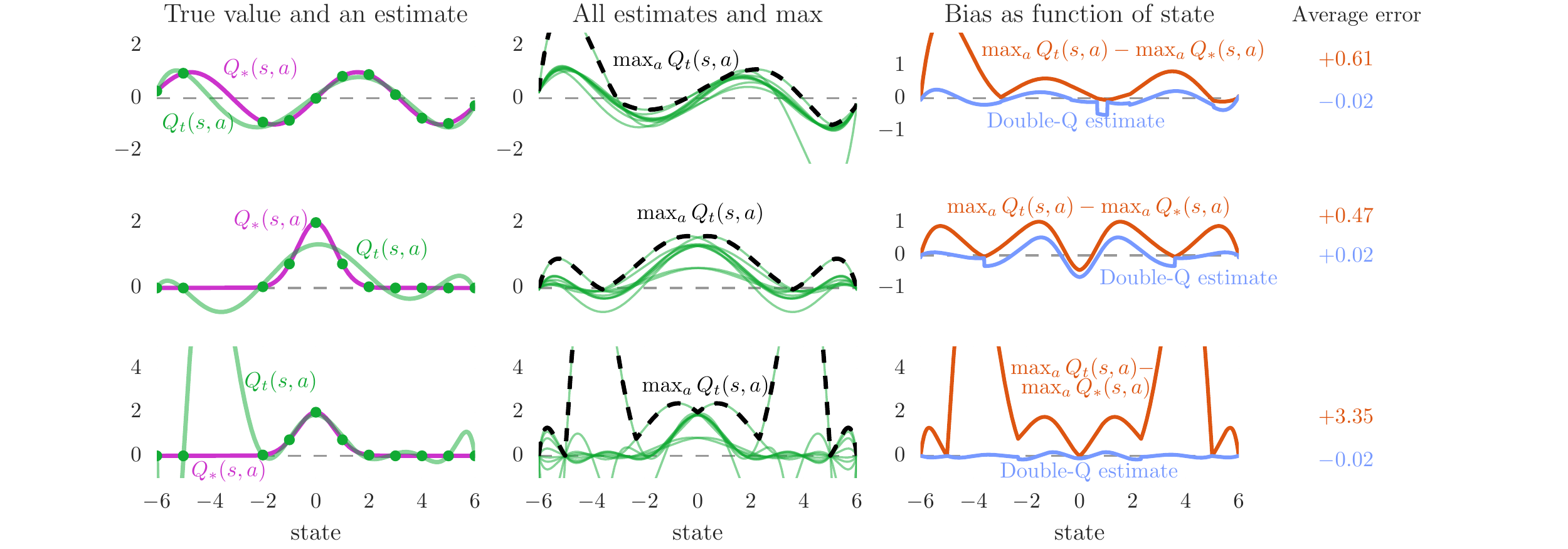}
\caption{\label{function_overest}
\small{
Illustration of overestimations during learning. In each state (x-axis), there are 10 actions. The \textbf{left column} shows the true values $V_*(s)$ (purple line). All true action values are defined by $Q_*(s,a) = V_*(s)$. The green line shows estimated values $Q(s,a)$ for one action as a function of state, fitted to the true value at several sampled states (green dots).
The \textbf{middle column} plots show all the estimated values (green), and the maximum of these values (dashed black). The maximum is higher than the true value (purple, left plot) almost everywhere. The \textbf{right column} plots shows the difference in orange. The blue line in the right plots is the estimate used by Double Q-learning with a second set of samples for each state. The blue line is much closer to zero, indicating less bias. The three \textbf{rows} correspond to different true functions (left, purple) or capacities of the fitted function (left, green). (Details in the text) 
}}
\end{figure*}

We now turn to function approximation and consider a real-valued continuous state space with 10 discrete actions in each state. For simplicity, the true optimal action values in this example depend only on state so that in each state all actions have the same true value. These true values are shown in the left column of plots in Figure \ref{function_overest} (purple lines) and are defined as either $Q_*(s,a) = \sin(s)$ (top row) or $Q_*(s,a) = 2 \exp(-s^2)$ (middle and bottom rows). The left plots also show an approximation for a single action (green lines) as a function of state as well as the samples the estimate is based on (green dots).  The estimate is a $d$-degree polynomial that is fit to the true values at sampled states, where $d=6$ (top and middle rows) or $d=9$ (bottom row). The samples match the true function exactly: there is no noise and we assume we have ground truth for the action value on these sampled states.  The approximation is inexact even on the sampled states for the top two rows because the function approximation is insufficiently flexible. In the bottom row, the function is flexible enough to fit the green dots, but this reduces the accuracy in unsampled states.   Notice that the sampled states are spaced further apart near the left side of the left plots, resulting in larger estimation errors.  In many ways this is a typical learning setting, where at each point in time we only have limited data.

The middle column of plots in Figure \ref{function_overest} shows estimated action value functions for all 10 actions (green lines), as functions of state, along with the maximum action value in each state (black dashed line).  Although the true value function is the same for all actions, the approximations differ because we have supplied different sets of sampled states.\footnote{Each action-value function is fit with a different subset of integer states. States $-6$ and $6$ are always included to avoid extrapolations, and for each action two adjacent integers are missing: for action $a_1$ states $-5$ and $-4$ are not sampled, for $a_2$ states $-4$ and $-3$ are not sampled, and so on. This causes the estimated values to differ.} The maximum is often higher than the ground truth shown in purple on the left.  This is confirmed in the right plots, which shows the difference between the black and purple curves in orange.  The orange line is almost always positive, indicating an upward bias. The right plots also show the estimates from Double Q-learning in blue\footnote{We arbitrarily used the samples of action $a_{i+5}$ (for $i\le 5$) or $a_{i-5}$ (for $i > 5$) as the second set of samples for the double estimator of action $a_i$.}, which are on average much closer to zero.  This demonstrates that Double Q-learning indeed can successfully reduce the overoptimism of Q-learning.

The different rows in Figure \ref{function_overest} show variations of the same experiment.  The difference between the top and middle rows is the true value function, demonstrating that overestimations are not an artifact of a specific true value function.  The difference between the middle and bottom rows is the flexibility of the function approximation.  In the left-middle plot, the estimates are even incorrect for some of the sampled states because the function is insufficiently flexible.  The function in the bottom-left plot is more flexible but this causes higher estimation errors for unseen states, resulting in higher overestimations.  This is important because flexible parametric function approximators are often employed in reinforcement learning (see, e.g., \citealt{Tesauro:1995,Sallans:2004,Riedmiller:2005,Mnih:2015}). 

In contrast to \citet{vanHasselt:2010} we did not use a statistical argument to find overestimations, the process to obtain Figure \ref{function_overest} is fully deterministic. In contrast to \citet{Thrun:1993}, we did not rely on inflexible function approximation with irreducible asymptotic errors; the bottom row shows that a function that is flexible enough to cover all samples leads to high overestimations. This indicates that the overestimations can occur quite generally.

In the examples above, overestimations occur even when assuming we have samples of the \textit{true} action value at certain states.
The value estimates can further deteriorate if we bootstrap off of action values that are already overoptimistic, since this causes overestimations to propagate throughout our estimates.  
Although \textit{uniformly} overestimating values might not hurt the resulting policy, in practice overestimation errors will differ for different states and actions. Overestimation combined with bootstrapping then has the pernicious effect of propagating the wrong relative information about which states are more valuable than others, directly affecting the quality of the learned policies. 

The overestimations should not be confused with optimism in the face of uncertainty \citep{Sutton:1990,Agrawal:1995,Kaelbling:1996,Auer:2002,Brafman:2003,Szita:2008,Strehl:2009}, where an exploration bonus is given to states or actions with uncertain values.  Conversely, the overestimations discussed here occur only after updating, resulting in overoptimism in the face of apparent certainty.  This was already observed by \citet{Thrun:1993}, who noted that, in contrast to optimism in the face of uncertainty, these overestimations actually can impede learning an optimal policy.  We will see this negative effect on policy quality confirmed later in the experiments as well: when we reduce the overestimations using Double Q-learning, the policies improve.

\section{Double DQN}

The idea of Double Q-learning is to reduce overestimations by decomposing the max operation 
in the target into action selection and action evaluation.
Although not fully decoupled, the target network in the DQN architecture provides a natural candidate for the second value function, without having to introduce additional networks. 
We therefore propose to evaluate the greedy policy according to the online network, but using the target network to estimate its value.
In reference to both Double Q-learning and DQN, we refer to the resulting algorithm as Double DQN. 
Its update is the same as for DQN, but replacing the target $Y^{\text{DQN}}_t$ with
\[
Y^\text{DoubleDQN}_t \equiv R_{t+1} + \g Q( S_{t+1}, \argmax_a Q( S_{t+1}, a ; \th_t ), \th^-_t ) \,.
\]
In comparison to Double Q-learning \eqref{TDDQ}, the weights of the second network $\th'_t$ are replaced with the weights of the target network $\th^-_t$ for the evaluation of the current greedy policy.  The update to the target network stays unchanged from DQN, and remains a periodic copy of the online network.

This version of Double DQN is perhaps the minimal possible change to DQN towards Double Q-learning.  The goal is to get most of the benefit of Double Q-learning, while keeping the rest of the DQN algorithm intact for a fair comparison, and with minimal computational overhead.

\section{Empirical results}
\label{sec:results}

In this section, we analyze the overestimations of DQN and show that Double DQN improves over DQN both in terms of value accuracy and in terms of policy quality. To further test the robustness of the approach we additionally evaluate  the algorithms with random starts generated from expert human trajectories, as proposed by \citet{Nair:2015}.

Our testbed consists of Atari 2600 games, using the Arcade Learning Environment~\citep{Bellemare:2013}. The goal is for a single algorithm, with a fixed set of hyperparameters, to learn to play each of the games separately from interaction given only the screen pixels as input. 
This is a demanding testbed: not only are the inputs high-dimensional, the game visuals and game mechanics vary substantially between games. Good solutions must therefore rely heavily on the learning algorithm --- it is not practically feasible to overfit the domain by relying only on tuning.

We closely follow the experimental setting and network architecture outlined by \citet{Mnih:2015}. Briefly, the network architecture is a convolutional neural network \citep{Fukushima:1988,Lecun:1998} with 3 convolution layers and a fully-connected hidden layer (approximately 1.5M parameters in total). The network takes the last four frames as input and outputs the action value of each action. On each game, the network is trained on a single GPU for 200M frames, or approximately 1 week.

\subsection{Results on overoptimism}
\label{sec:result:overopt}
\begin{figure*}[t]
\begin{center}
\includegraphics[width=0.7\textwidth]{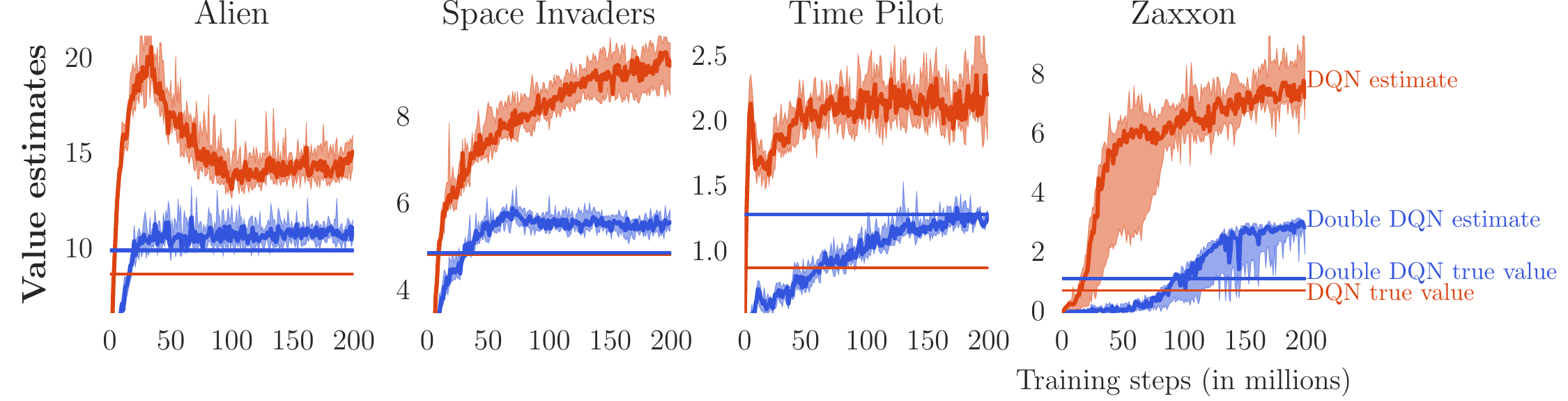}
\includegraphics[width=0.7\textwidth]{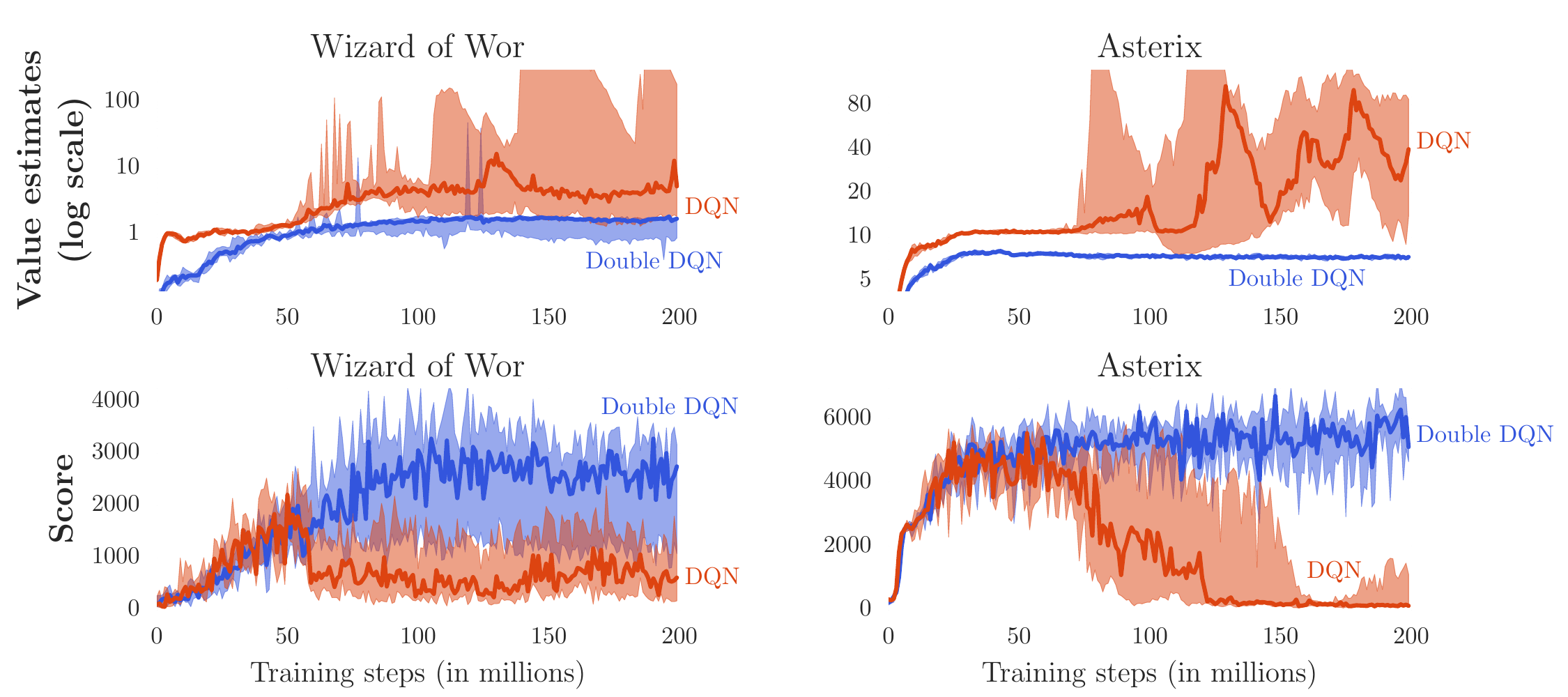}
\caption{\label{fig:dqn_overest} \small{The \textbf{top} and \textbf{middle} rows show value estimates by DQN (orange) and Double DQN (blue) on six Atari games. The results are obtained by running DQN and Double DQN with 6 different random seeds with the hyper-parameters employed by \citet{Mnih:2015}. The darker line shows the median over seeds and we average the two extreme values to obtain the shaded area (i.e., 10\% and 90\% quantiles with linear interpolation).
The straight horizontal orange (for DQN) and blue (for Double DQN) lines in the top row are computed by running the corresponding agents after learning concluded, and averaging the actual discounted return obtained from each visited state.  These straight lines would match the learning curves at the right side of the plots if there is no bias.
The \textbf{middle} row shows the value estimates (in log scale) for two games in which DQN's overoptimism is quite extreme.
The \textbf{bottom} row shows the detrimental effect of this on the score achieved by the agent as it is evaluated during training: the scores drop when the overestimations begin. Learning with Double DQN is much more stable.}}
\end{center}
\end{figure*}

Figure~\ref{fig:dqn_overest} shows examples of DQN's overestimations in six Atari games. DQN and Double DQN were both trained under the exact conditions described by \citet{Mnih:2015}.  DQN is consistently and sometimes vastly overoptimistic about the value of the current greedy policy, as can be seen by comparing the orange learning curves in the top row of plots to the straight orange lines, which represent the actual discounted value of the best learned policy.  
More precisely, the (averaged) value estimates are computed regularly during training with full evaluation phases of length $T=125,000$ steps as
\[
\frac{1}{T} \sum_{t=1}^T \argmax_a Q( S_{t}, a ; \th) \,.
\]The ground truth averaged values are obtained by running the best learned policies for several episodes and computing the actual cumulative rewards.      
Without overestimations we would expect these quantities to match up (i.e., the curve to match the straight line at the right of each plot).  Instead, the learning curves of DQN consistently end up much higher than the true values.  The learning curves for Double DQN, shown in blue, are much closer to the blue straight line representing the true value of the final policy.  Note that the blue straight line is often higher than the orange straight line.  This indicates that Double DQN does not just produce more accurate value estimates but also better policies.

More extreme overestimations are shown in the middle two plots, where DQN is highly unstable on the games Asterix and Wizard of Wor.  Notice the log scale for the values on the $y$-axis.  The bottom two plots shows the corresponding scores for these two games.  Notice that the increases in value estimates for DQN in the middle plots coincide with decreasing scores in bottom plots.  Again, this indicates that the overestimations are harming the quality of the resulting policies.
If seen in isolation, one might perhaps be tempted to think the observed instability is related to inherent instability problems of off-policy learning with function approximation \citep{Baird:1995,Tsitsiklis:1997,Sutton:2008,Maei:2011,Sutton:2015}.  However, we see that learning is much more stable with Double DQN, suggesting that the cause for these instabilities is in fact Q-learning's overoptimism.
Figure \ref{fig:dqn_overest} only shows a few examples, but overestimations were observed for DQN in all 49 tested Atari games, albeit in varying amounts.

\subsection{Quality of the learned policies}
\label{sec:result:quality}
\begin{table}[t]
\centering
 \def\arraystretch{1.1}
 \begin{tabular}{|l|rr|}
 \hline
  		        &\small DQN & \small Double DQN 		\\ \hline 
 \small Median  &\small ~93.5\% 	& \small 114.7\%~~       \\ \hline 
 \small Mean    &\small 241.1\% 	& \small 330.3\%~~ 	    \\ \hline 
  \end{tabular}
\caption{\small Summary of normalized performance up to 5 minutes of play on 49 games.  Results for DQN are from \citet{Mnih:2015}}
\label{tableresults}
\end{table}

Overoptimism does not always adversely affect the quality of the learned policy.
For example, DQN achieves optimal behavior in Pong despite slightly overestimating the policy value.
Nevertheless, reducing overestimations can significantly benefit the stability of learning; we see clear examples of this in Figure~\ref{fig:dqn_overest}. We now assess more generally how much Double DQN helps in terms of policy quality by evaluating on all 49 games that DQN was tested on.

As described by \citet{Mnih:2015} each evaluation episode starts by executing a special
no-op action that does not affect the environment up to 30 times, to provide different starting points for the agent.
Some exploration
during evaluation provides additional randomization.
For Double DQN we used the exact same hyper-parameters as for DQN, to allow for a controlled experiment focused just on reducing overestimations.  The learned policies are evaluated for 5 mins of emulator time (18,000 frames)
with an $\epsilon$-greedy policy where $\epsilon= 0.05$. The scores are averaged over 100 episodes.
The only difference between Double DQN and DQN is the target, using $Y^{\text{DoubleDQN}}_t$ rather than $Y^{\text{DQN}}$.
This evaluation is somewhat adversarial, as the used hyper-parameters were tuned for DQN but not for Double DQN.

To obtain summary statistics across games, we normalize the score for each game as follows:
\begin{equation}
  \text{score}_{\text{normalized}} = \frac{\text{score}_\text{agent} - \text{score}_\text{random}}{\text{score}_\text{human} - \text{score}_\text{random}}.
\end{equation}
The `random' and `human' scores are the same as used by \citet{Mnih:2015}, and are given in the appendix.

Table~\ref{tableresults}, under \textbf{no ops}, shows that on the whole Double DQN clearly improves over DQN.
A detailed comparison (in appendix) shows that there are several games in which Double DQN greatly improves upon DQN.  Noteworthy examples include Road Runner (from 233\% to 617\%), Asterix (from 70\% to 180\%), Zaxxon (from 54\% to 111\%), and Double Dunk (from 17\% to 397\%).

The Gorila algorithm \citep{Nair:2015}, which is a massively distributed version of DQN, is not included in the table because the architecture and infrastructure is sufficiently different to make a direct comparison unclear.  For completeness, we note that Gorila obtained median and mean normalized scores of 96\% and 495\%, respectively.

\subsection{Robustness to Human starts}
\label{sec:result:human_starts}
\begin{table}[t]
\centering
 \def\arraystretch{1.1}
 \begin{tabular}{|l|rrr|}
  \hline
  		        &\small DQN  & \small Double DQN  	& \small Double DQN (tuned) \\ \hline 
 \small Median  &\small ~47.5\% 	& \small ~88.4\%~~      & \small 116.7\%~~       \\ \hline 
 \small Mean    &\small 122.0\% 	& \small 273.1\%~~ 	    & \small 475.2\%~~ 	    \\ \hline 
 \end{tabular}
\caption{\small Summary of normalized performance up to 30 minutes of play on 49 games with human starts. Results for DQN are from \citet{Nair:2015}.}
\label{tableresults_at30_rnd_starts}
\end{table}

One concern with the previous evaluation is that in deterministic games with a
unique starting point the learner could potentially learn to remember sequences of actions without
much need to generalize. While successful, the solution would
not be particularly robust.  By testing the agents from various starting points,
we can test whether the found solutions generalize well, and as such provide a
challenging testbed for the learned polices~\citep{Nair:2015}.

We obtained 100 starting points sampled for each game from a human
expert's trajectory, as proposed by \citet{Nair:2015}. We start an
evaluation episode from each of these starting points and run the emulator for
up to 108,000 frames (30 mins at 60Hz including the trajectory before the
starting point). Each agent is only evaluated on the rewards accumulated after
the starting point.
\begin{figure}[t]
\begin{center}
\includegraphics[width=3.1in]{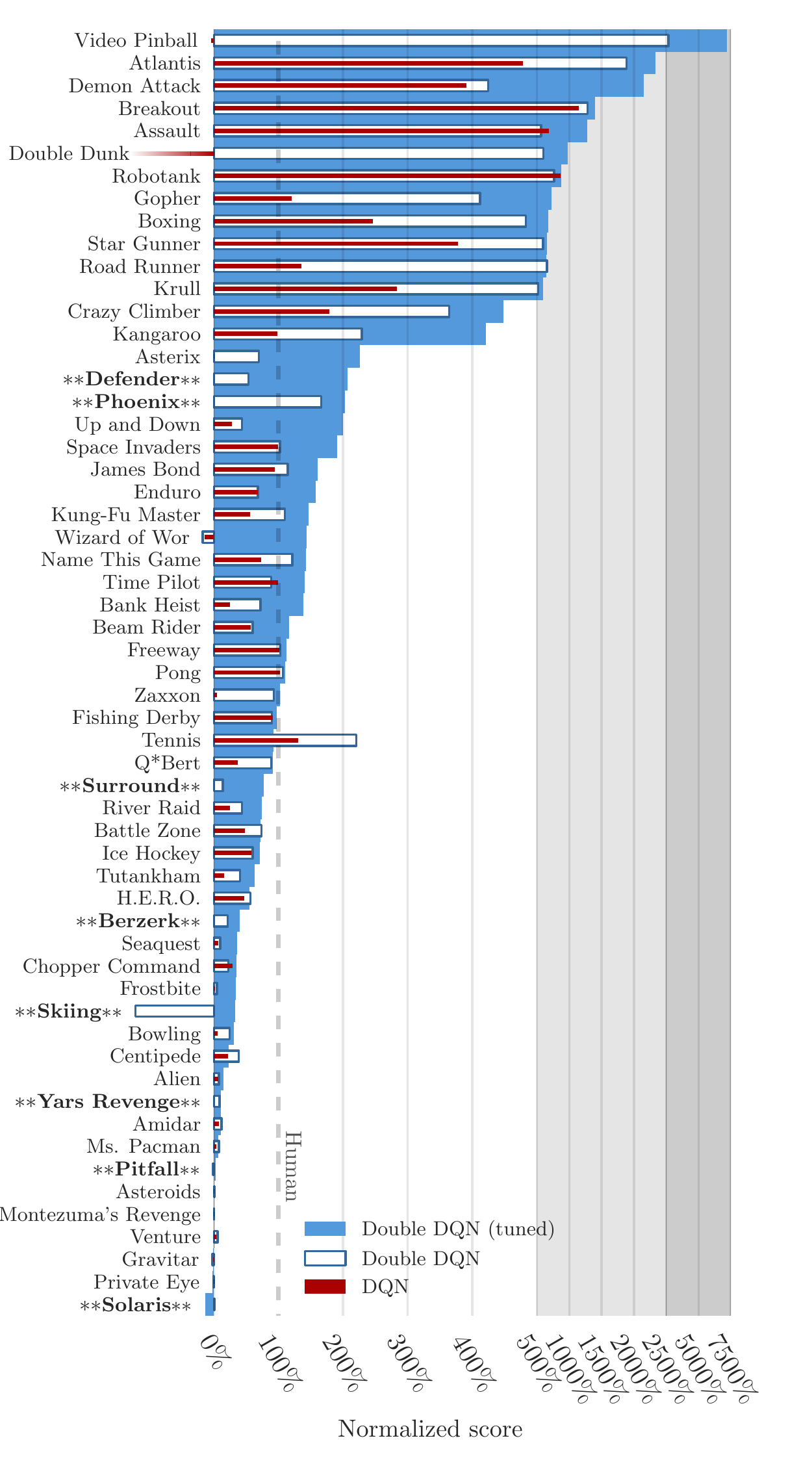}
\caption{ \label{at30_rnd_fig} {Normalized scores on 57 Atari games, tested for 100 episodes per game with human starts.  Compared to \citet{Mnih:2015}, eight games additional games were tested.  These are indicated with stars and a bold font.}}
\end{center}
\end{figure}

For this evaluation we include a tuned version of Double DQN.  Some tuning is appropriate because the hyperparameters were tuned for DQN, which is a different algorithm.  For the tuned version of Double DQN, we increased the number of frames between each two copies of the target network from 10,000 to 30,000, to reduce overestimations further because immediately after each switch DQN and Double DQN both revert to Q-learning.  In addition, we reduced the exploration during learning from $\epsilon=0.1$ to $\epsilon=0.01$, and then used $\epsilon=0.001$ during evaluation.  Finally, the tuned version uses a single shared bias for all action values in the top layer of the network.  Each of these changes improved performance and together they result in clearly better results.\footnote{Except for Tennis, where the lower $\epsilon$ during training seemed to hurt rather than help.}

Table~\ref{tableresults_at30_rnd_starts} reports summary statistics for this evaluation on the 49 games from \citet{Mnih:2015}.  Double DQN obtains clearly higher median and mean scores.  Again Gorila DQN \citep{Nair:2015} is not included in the table, but for completeness note it obtained a median of 78\% and a mean of 259\%. 
Detailed results, plus results for an additional 8 games, are available in Figure~\ref{at30_rnd_fig} and in the appendix. On several games the improvements from DQN to Double DQN are striking, in some cases bringing scores much closer to human, or even surpassing these.

Double DQN appears more robust to this more challenging evaluation,
suggesting that appropriate generalizations occur and that the found solutions do not exploit the determinism of the environments.
This is appealing, as it indicates progress towards finding general solutions rather than a deterministic sequence of steps that would be less robust.

\section{Discussion}

This paper has five contributions.  First, we have shown why Q-learning can be overoptimistic in large-scale problems, even if these are deterministic, due to the inherent estimation errors of learning.  Second, by analyzing the value estimates on Atari games we have shown that these overestimations are more common and severe in practice than previously acknowledged.
Third, we have shown that Double Q-learning can be used at scale to successfully reduce this overoptimism, resulting in more stable and reliable learning.  Fourth, we have proposed a specific implementation called Double DQN, that uses the existing architecture and deep neural network of the DQN algorithm without requiring additional networks or parameters.  Finally, we have shown that Double DQN finds better policies, obtaining new state-of-the-art results on the Atari 2600 domain.

\section*{Acknowledgments}
We would like to thank Tom Schaul, Volodymyr Mnih, Marc Bellemare, Thomas Degris, Georg Ostrovski, and Richard Sutton for helpful comments, and everyone at Google DeepMind for a constructive research environment.
\small
\bibliography{aaai}
\bibliographystyle{abbrvnat}
\newpage
\section*{Appendix}
\setcounter{theorem}{0}
\begin{theorem}
Consider a state $s$ in which all the true optimal action values are equal at $Q_*(s,a) = V_*(s)$ for some $V_*(s)$. Let $Q_t$ be arbitrary value estimates that are on the whole unbiased in the sense that $\sum_a ( Q_t(s,a) - V_*(s) ) = 0$, but that are not all zero, such that $\frac{1}{m} \sum_a ( Q_t(s,a) - V_*(s) )^2 = C$ for some $C > 0$, where $m \ge 2$ is the number of actions in $s$. Under these conditions, $\max_a Q_t(s,a) \ge V_*(s) + \sqrt{\frac{C}{m-1}}$.  This lower bound is tight. Under the same conditions, the lower bound on the absolute error of the Double Q-learning estimate is zero.
\end{theorem}
\begin{proof}[Proof of Theorem 1]
Define the errors for each action $a$ as $\epsilon_a = Q_t(s,a) - V_*(s)$. Suppose that there exists a setting of $\{ \epsilon_a \}$ such that $\max_a \epsilon_a < \sqrt{\frac{C}{m-1}}$. Let $\{ \epsilon^+_i \}$ be the set of positive $\epsilon$ of size $n$, and $\{ \epsilon^-_j \}$ the set of strictly negative $\epsilon$ of size $m-n$ (such that $\{ \epsilon\} = \{ \epsilon^+_i \} \cup \{ \epsilon^-_j \}$). If $n=m$, then $\sum_a \epsilon_a = 0 \implies \epsilon_a = 0 \;\forall a$, which contradicts $\sum_a \epsilon_a^2 = m C$. Hence, it must be that $n \leq m-1$. Then, $\sum_{i=1}^n \epsilon^+_i \le n \max_i \epsilon^+_i < n \sqrt{\frac{C}{m-1}}$, and therefore (using the constraint $\sum_a \epsilon_a = 0$) we also have that $\sum_{j=1}^{m-n} |\epsilon^-_j| < n \sqrt{\frac{C}{m-1}}$. This implies $\max_j |\epsilon^-_j| < n \sqrt{\frac{C}{m-1}}$. 
By H\"{o}lder's inequality, then
\begin{align*}
\sum_{j=1}^{m-n} (\epsilon^-_j)^2
&\leq \sum_{j=1}^{m-n} |\epsilon^-_j| \cdot \max_j |\epsilon^-_j| \\
&< n \sqrt{\frac{C}{m-1}} n \sqrt{\frac{C}{m-1}}\,.
\end{align*}
We can now combine these relations to compute an upper-bound on the sum of squares for all $\epsilon_a$:
\begin{align*}
\sum_{a=1}^m (\epsilon_a)^2 &= \sum_{i=1}^{n} (\epsilon^+_i)^2 + \sum_{j=1}^{m-n} (\epsilon^-_j)^2 \\ 
&< n \frac{C}{m-1} +  n \sqrt{\frac{C}{m-1}} n \sqrt{\frac{C}{m-1}} \\
&= C \frac{n(n+1)}{m-1}\\
&\leq m C.
\end{align*}
This contradicts the assumption that $\sum_{a=1}^m \epsilon_a^2 < m C$, and therefore $\max_a \epsilon_a \geq \sqrt{\frac{C}{m-1}}$ for all settings of $\epsilon$ that satisfy the constraints. We can check that the lower-bound is tight by setting $\epsilon_a = \sqrt{\frac{C}{m-1}}$ for $a=1,\dots,m-1$ and $\epsilon_m = - \sqrt{(m-1)C}$. This verifies $\sum_a \epsilon_a^2 = m C$ and $\sum_a \epsilon_a = 0$.

The only tight lower bound on the absolute error for Double Q-learning $|Q'_t(s,\argmax_a Q_t(s,a) ) - V_*(s)|$ is zero.
This can be seen by because we can have
\[
Q_t(s,a_1) = V_*(s) + \sqrt{C\frac{m-1}{m}}\,,
\]
and
\[
Q_t(s,a_i) = V_*(s)-\sqrt{C\frac{1}{m(m-1)}}\,\text{, for $i > 1$.}
\]
Then the conditions of the theorem hold.  If then, furthermore, we have $Q'_t(s,a_1) = V_*(s)$ then the error is zero.  The remaining action values $Q'_t(s,a_i)$, for $i>1$, are arbitrary.
\end{proof}

\begin{theorem}
Consider a state $s$ in which all the true optimal action values are equal at $Q_*(s,a) = V_*(s)$.  Suppose that the estimation errors $Q_t(s,a) - Q_*(s,a)$ are independently distributed uniformly randomly in $[-1,1]$.  Then,
\begin{align*}
\E{ \max_a Q_t(s,a) - V_*(s) } = \frac{m-1}{m+1} \,
\end{align*}
\end{theorem}
\begin{proof}
Define $\epsilon_a = Q_t(s,a) - Q_*(s,a)$; this is a uniform random variable in $[-1,1]$.  The probability that $\max_a Q_t(s,a) \le x$ for some $x$ is equal to the probability that $\epsilon_a \le x$ for all $a$ simultaneously.  Because the estimation errors are independent, we can derive
\begin{align*}
P( \max_a \epsilon_a \le x )
& = P( X_1 \le x \land X_2 \le x \land \ldots \land X_m \le x )\\
& = \prod_{a=1}^m P( \epsilon_a \le x )\,.
\end{align*}
The function $P( \epsilon_a \le x )$ is the cumulative distribution function (CDF) of $\epsilon_a$, which here is simply defined as
\[
P( \epsilon_a \le x ) = 
\left\{
\begin{array}{ll}
0 & \text{if $x \le -1$} \\
\frac{1+x}{2} & \text{if $x \in (-1,1)$} \\
1 & \text{if $x \ge 1$} \\
\end{array}
\right.
\]
This implies that
\begin{align*}
P( \max_a \epsilon_a \le x )
& = \prod_{a=1}^m P( \epsilon_a \le x )\\
& = 
\left\{
\begin{array}{ll}
0 & \text{if $x \le -1$} \\
\left(\frac{1+x}{2}\right)^m & \text{if $x \in (-1,1)$} \\
1 & \text{if $x \ge 1$} \\
\end{array}
\right.
\end{align*}
This gives us the CDF of the random variable $\max_a \epsilon_a$.  Its expectation can be written as an integral
\[
\E{ \max_a \epsilon_a } = \int_{-1}^{1} \! x f_{\max}(x)\,\text{d}x \,,
\]
where $f_{\max}$ is the probability density function of this variable, defined as the derivative of the CDF: $f_{\max}(x) = \frac{\text{d}}{\text{d}x} P( \max_a \epsilon_a \le x )$, so that for $x \in [-1,1]$ we have $f_{\max}(x) = \frac{m}{2} \left(\frac{1+x}{2}\right)^{m-1}$.  Evaluating the integral yields
\begin{align*}
\E{ \max_a \epsilon_a }
&= \int_{-1}^{1} \! x f_{\max}(x)\,\text{d}x \\
&= \left[ \left(\frac{x+1}{2}\right)^m \frac{m x - 1}{m+1} \right]_{-1}^1 \\
&= \frac{m-1}{m+1} \,. \qedhere
\end{align*}
\end{proof}

\section*{Experimental Details for the Atari 2600 Domain}

We selected the 49 games to match the list used by \citet{Mnih:2015}, see Tables below for the full list. Each agent step is composed of four frames (the last selected action is repeated during these frames) and reward values (obtained from the Arcade Learning Environment \citep{Bellemare:2013}) are clipped between -1 and 1. 

\subsection*{Network Architecture}

The convolution network used in the experiment is exactly the one proposed by proposed by \citet{Mnih:2015}, we only provide details here for completeness. The input to the network is a 84x84x4 tensor containing a rescaled, and gray-scale, version of the last four frames. The first convolution layer convolves the input with 32 filters of size 8 (stride 4), the second layer has 64 layers of size 4 (stride 2), the final convolution layer has 64 filters of size 3 (stride 1). 
This is followed by a fully-connected hidden layer of 512 units.
All these layers are separated by Rectifier Linear Units (ReLu). Finally, a fully-connected linear layer projects to the output of the network, i.e., the Q-values.
The optimization employed to train the network is RMSProp (with momentum parameter $0.95$). 

\subsection*{Hyper-parameters}

In all experiments, the discount was set to $\gamma=0.99$, and the learning rate to $\alpha=0.00025$. 
The number of steps between target network updates was $\tau=10,000$. Training is done over 50M steps (i.e., 200M frames). The agent is evaluated every 1M steps, and the best policy across these evaluations is kept as the output of the learning process. 
The size of the experience replay memory is 1M tuples. The memory gets sampled to update the network every 4 steps with minibatches of size 32.
The simple exploration policy used is an $\epsilon$-greedy policy with the $\epsilon$ decreasing linearly from 1 to $0.1$ over 1M steps.

\section*{Supplementary Results in the Atari 2600 Domain}
The Tables below provide further detailed results for our experiments in the Atari domain.

\begin{table*}[h]
\centering
\begin{tabular}{lrrrrr}
\textbf{Game} & \textbf{Random} & \textbf{Human} & \textbf{DQN} & \textbf{Double DQN} \\
Alien & 227.80 & 6875.40 & 3069.33 & 2907.30 \\
Amidar & 5.80 & 1675.80 & 739.50 & 702.10 \\
Assault & 222.40 & 1496.40 & 3358.63 & 5022.90 \\
Asterix & 210.00 & 8503.30 & 6011.67 & 15150.00 \\
Asteroids & 719.10 & 13156.70 & 1629.33 & 930.60 \\
Atlantis & 12850.00 & 29028.10 & 85950.00 & 64758.00 \\
Bank Heist & 14.20 & 734.40 & 429.67 & 728.30 \\
Battle Zone & 2360.00 & 37800.00 & 26300.00 & 25730.00 \\
Beam Rider & 363.90 & 5774.70 & 6845.93 & 7654.00 \\
Bowling & 23.10 & 154.80 & 42.40 & 70.50 \\
Boxing & 0.10 & 4.30 & 71.83 & 81.70 \\
Breakout & 1.70 & 31.80 & 401.20 & 375.00 \\
Centipede & 2090.90 & 11963.20 & 8309.40 & 4139.40 \\
Chopper Command & 811.00 & 9881.80 & 6686.67 & 4653.00 \\
Crazy Climber & 10780.50 & 35410.50 & 114103.33 & 101874.00 \\
Demon Attack & 152.10 & 3401.30 & 9711.17 & 9711.90 \\
Double Dunk & -18.60 & -15.50 & -18.07 & -6.30 \\
Enduro & 0.00 & 309.60 & 301.77 & 319.50 \\
Fishing Derby & -91.70 & 5.50 & -0.80 & 20.30 \\
Freeway & 0.00 & 29.60 & 30.30 & 31.80 \\
Frostbite & 65.20 & 4334.70 & 328.33 & 241.50 \\
Gopher & 257.60 & 2321.00 & 8520.00 & 8215.40 \\
Gravitar & 173.00 & 2672.00 & 306.67 & 170.50 \\
H.E.R.O. & 1027.00 & 25762.50 & 19950.33 & 20357.00 \\
Ice Hockey & -11.20 & 0.90 & -1.60 & -2.40 \\
James Bond & 29.00 & 406.70 & 576.67 & 438.00 \\
Kangaroo & 52.00 & 3035.00 & 6740.00 & 13651.00 \\
Krull & 1598.00 & 2394.60 & 3804.67 & 4396.70 \\
Kung-Fu Master & 258.50 & 22736.20 & 23270.00 & 29486.00 \\
Montezuma's Revenge & 0.00 & 4366.70 & 0.00 & 0.00 \\
Ms. Pacman & 307.30 & 15693.40 & 2311.00 & 3210.00 \\
Name This Game & 2292.30 & 4076.20 & 7256.67 & 6997.10 \\
Pong & -20.70 & 9.30 & 18.90 & 21.00 \\
Private Eye & 24.90 & 69571.30 & 1787.57 & 670.10 \\
Q*Bert & 163.90 & 13455.00 & 10595.83 & 14875.00 \\
River Raid & 1338.50 & 13513.30 & 8315.67 & 12015.30 \\
Road Runner & 11.50 & 7845.00 & 18256.67 & 48377.00 \\
Robotank & 2.20 & 11.90 & 51.57 & 46.70 \\
Seaquest & 68.40 & 20181.80 & 5286.00 & 7995.00 \\
Space Invaders & 148.00 & 1652.30 & 1975.50 & 3154.60 \\
Star Gunner & 664.00 & 10250.00 & 57996.67 & 65188.00 \\
Tennis & -23.80 & -8.90 & -2.47 & 1.70 \\
Time Pilot & 3568.00 & 5925.00 & 5946.67 & 7964.00 \\
Tutankham & 11.40 & 167.60 & 186.70 & 190.60 \\
Up and Down & 533.40 & 9082.00 & 8456.33 & 16769.90 \\
Venture & 0.00 & 1187.50 & 380.00 & 93.00 \\
Video Pinball & 16256.90 & 17297.60 & 42684.07 & 70009.00 \\
Wizard of Wor & 563.50 & 4756.50 & 3393.33 & 5204.00 \\
Zaxxon & 32.50 & 9173.30 & 4976.67 & 10182.00 \\
\end{tabular}
\caption{Raw scores for the no-op evaluation condition (5 minutes emulator time). DQN as given by \citet{Mnih:2015}.}
\end{table*}

\begin{table*}[h]
\centering
\begin{tabular}{lrrr}
\textbf{Game} & \textbf{DQN} & \textbf{Double DQN} \\
Alien & 42.75 \% & 40.31 \% \\
Amidar & 43.93 \% & 41.69 \% \\
Assault & 246.17 \% & 376.81 \% \\
Asterix & 69.96 \% & 180.15 \% \\
Asteroids & 7.32 \% & 1.70 \% \\
Atlantis & 451.85 \% & 320.85 \% \\
Bank Heist & 57.69 \% & 99.15 \% \\
Battle Zone & 67.55 \% & 65.94 \% \\
Beam Rider & 119.80 \% & 134.73 \% \\
Bowling & 14.65 \% & 35.99 \% \\
Boxing & 1707.86 \% & 1942.86 \% \\
Breakout & 1327.24 \% & 1240.20 \% \\
Centipede & 62.99 \% & 20.75 \% \\
Chopper Command & 64.78 \% & 42.36 \% \\
Crazy Climber & 419.50 \% & 369.85 \% \\
Demon Attack & 294.20 \% & 294.22 \% \\
Double Dunk & 17.10 \% & 396.77 \% \\
Enduro & 97.47 \% & 103.20 \% \\
Fishing Derby & 93.52 \% & 115.23 \% \\
Freeway & 102.36 \% & 107.43 \% \\
Frostbite & 6.16 \% & 4.13 \% \\
Gopher & 400.43 \% & 385.66 \% \\
Gravitar & 5.35 \% & -0.10 \% \\
H.E.R.O. & 76.50 \% & 78.15 \% \\
Ice Hockey & 79.34 \% & 72.73 \% \\
James Bond & 145.00 \% & 108.29 \% \\
Kangaroo & 224.20 \% & 455.88 \% \\
Krull & 277.01 \% & 351.33 \% \\
Kung-Fu Master & 102.37 \% & 130.03 \% \\
Montezuma's Revenge & 0.00 \% & 0.00 \% \\
Ms. Pacman & 13.02 \% & 18.87 \% \\
Name This Game & 278.29 \% & 263.74 \% \\
Pong & 132.00 \% & 139.00 \% \\
Private Eye & 2.53 \% & 0.93 \% \\
Q*Bert & 78.49 \% & 110.68 \% \\
River Raid & 57.31 \% & 87.70 \% \\
Road Runner & 232.91 \% & 617.42 \% \\
Robotank & 508.97 \% & 458.76 \% \\
Seaquest & 25.94 \% & 39.41 \% \\
Space Invaders & 121.49 \% & 199.87 \% \\
Star Gunner & 598.09 \% & 673.11 \% \\
Tennis & 143.15 \% & 171.14 \% \\
Time Pilot & 100.92 \% & 186.51 \% \\
Tutankham & 112.23 \% & 114.72 \% \\
Up and Down & 92.68 \% & 189.93 \% \\
Venture & 32.00 \% & 7.83 \% \\
Video Pinball & 2539.36 \% & 5164.99 \% \\
Wizard of Wor & 67.49 \% & 110.67 \% \\
Zaxxon & 54.09 \% & 111.04 \% \\
\end{tabular}
\caption{\label{table:at5_norm} Normalized results for no-op evaluation condition (5 minutes emulator time).}
\end{table*}

\begin{table*}[h]
\centering
\begin{tabular}{lrrrrr}
\textbf{Game} & \textbf{Random} & \textbf{Human} & \textbf{DQN} & \textbf{Double DQN}  & \textbf{Double DQN (tuned)} \\
Alien & 128.30 & 6371.30 & 570.2 & 621.6 & 1033.4 \\
Amidar & 11.80 & 1540.40 & 133.4 & 188.2 & 169.1 \\
Assault & 166.90 & 628.90 & 3332.3 & 2774.3 & 6060.8 \\
Asterix & 164.50 & 7536.00 & 124.5 & 5285.0 & 16837.0 \\
Asteroids & 871.30 & 36517.30 & 697.1 & 1219.0 & 1193.2 \\
Atlantis & 13463.00 & 26575.00 & 76108.0 & 260556.0 & 319688.0 \\
Bank Heist & 21.70 & 644.50 & 176.3 & 469.8 & 886.0 \\
Battle Zone & 3560.00 & 33030.00 & 17560.0 & 25240.0 & 24740.0 \\
Beam Rider & 254.60 & 14961.00 & 8672.4 & 9107.9 & 17417.2 \\
Berzerk & 196.10 & 2237.50 &  & 635.8 & 1011.1 \\
Bowling & 35.20 & 146.50 & 41.2 & 62.3 & 69.6 \\
Boxing & -1.50 & 9.60 & 25.8 & 52.1 & 73.5 \\
Breakout & 1.60 & 27.90 & 303.9 & 338.7 & 368.9 \\
Centipede & 1925.50 & 10321.90 & 3773.1 & 5166.6 & 3853.5 \\
Chopper Command & 644.00 & 8930.00 & 3046.0 & 2483.0 & 3495.0 \\
Crazy Climber & 9337.00 & 32667.00 & 50992.0 & 94315.0 & 113782.0 \\
Defender & 1965.50 & 14296.00 &  & 8531.0 & 27510.0 \\
Demon Attack & 208.30 & 3442.80 & 12835.2 & 13943.5 & 69803.4 \\
Double Dunk & -16.00 & -14.40 & -21.6 & -6.4 & -0.3 \\
Enduro & -81.80 & 740.20 & 475.6 & 475.9 & 1216.6 \\
Fishing Derby & -77.10 & 5.10 & -2.3 & -3.4 & 3.2 \\
Freeway & 0.10 & 25.60 & 25.8 & 26.3 & 28.8 \\
Frostbite & 66.40 & 4202.80 & 157.4 & 258.3 & 1448.1 \\
Gopher & 250.00 & 2311.00 & 2731.8 & 8742.8 & 15253.0 \\
Gravitar & 245.50 & 3116.00 & 216.5 & 170.0 & 200.5 \\
H.E.R.O. & 1580.30 & 25839.40 & 12952.5 & 15341.4 & 14892.5 \\
Ice Hockey & -9.70 & 0.50 & -3.8 & -3.6 & -2.5 \\
James Bond & 33.50 & 368.50 & 348.5 & 416.0 & 573.0 \\
Kangaroo & 100.00 & 2739.00 & 2696.0 & 6138.0 & 11204.0 \\
Krull & 1151.90 & 2109.10 & 3864.0 & 6130.4 & 6796.1 \\
Kung-Fu Master & 304.00 & 20786.80 & 11875.0 & 22771.0 & 30207.0 \\
Montezuma's Revenge & 25.00 & 4182.00 & 50.0 & 30.0 & 42.0 \\
Ms. Pacman & 197.80 & 15375.00 & 763.5 & 1401.8 & 1241.3 \\
Name This Game & 1747.80 & 6796.00 & 5439.9 & 7871.5 & 8960.3 \\
Phoenix & 1134.40 & 6686.20 &  & 10364.0 & 12366.5 \\
Pit Fall & -348.80 & 5998.90 &  & -432.9 & -186.7 \\
Pong & -18.00 & 15.50 & 16.2 & 17.7 & 19.1 \\
Private Eye & 662.80 & 64169.10 & 298.2 & 346.3 & -575.5 \\
Q*Bert & 183.00 & 12085.00 & 4589.8 & 10713.3 & 11020.8 \\
River Raid & 588.30 & 14382.20 & 4065.3 & 6579.0 & 10838.4 \\
Road Runner & 200.00 & 6878.00 & 9264.0 & 43884.0 & 43156.0 \\
Robotank & 2.40 & 8.90 & 58.5 & 52.0 & 59.1 \\
Seaquest & 215.50 & 40425.80 & 2793.9 & 4199.4 & 14498.0 \\
Skiing & -15287.40 & -3686.60 &  & -29404.3 & -11490.4 \\
Solaris & 2047.20 & 11032.60 &  & 2166.8 & 810.0 \\
Space Invaders & 182.60 & 1464.90 & 1449.7 & 1495.7 & 2628.7 \\
Star Gunner & 697.00 & 9528.00 & 34081.0 & 53052.0 & 58365.0 \\
Surround & -9.70 & 5.40 &  & -7.6 & 1.9 \\
Tennis & -21.40 & -6.70 & -2.3 & 11.0 & -7.8 \\
Time Pilot & 3273.00 & 5650.00 & 5640.0 & 5375.0 & 6608.0 \\
Tutankham & 12.70 & 138.30 & 32.4 & 63.6 & 92.2 \\
Up and Down & 707.20 & 9896.10 & 3311.3 & 4721.1 & 19086.9 \\
Venture & 18.00 & 1039.00 & 54.0 & 75.0 & 21.0 \\
Video Pinball & 20452.0 & 15641.10 & 20228.1 & 148883.6 & 367823.7 \\
Wizard of Wor & 804.00 & 4556.00 & 246.0 & 155.0 & 6201.0 \\
Yars Revenge & 1476.90 & 47135.20 &  & 5439.5 & 6270.6 \\
Zaxxon & 475.00 & 8443.00 & 831.0 & 7874.0 & 8593.0 \\
\end{tabular}
\caption{Raw scores for the human start condition (30 minutes emulator time). DQN as given by \citet{Nair:2015}.}
\end{table*}

\begin{table*}[h]
\centering
\begin{tabular}{lrrr}
\textbf{Game} & \textbf{DQN} & \textbf{Double DQN} & \textbf{Double DQN (tuned)} \\
Alien & 7.08\% & 7.90\% & 14.50\% \\
Amidar & 7.95\% & 11.54\% & 10.29\% \\
Assault & 685.15\% & 564.37\% & 1275.74\% \\
Asterix & -0.54\% & 69.46\% & 226.18\% \\
Asteroids & -0.49\% & 0.98\% & 0.90\% \\
Atlantis & 477.77\% & 1884.48\% & 2335.46\% \\
Bank Heist & 24.82\% & 71.95\% & 138.78\% \\
Battle Zone & 47.51\% & 73.57\% & 71.87\% \\
Beam Rider & 57.24\% & 60.20\% & 116.70\% \\
Berzerk &  & 21.54\% & 39.92\% \\
Bowling & 5.39\% & 24.35\% & 30.91\% \\
Boxing & 245.95\% & 482.88\% & 675.68\% \\
Breakout & 1149.43\% & 1281.75\% & 1396.58\% \\
Centipede & 22.00\% & 38.60\% & 22.96\% \\
Chopper Command & 28.99\% & 22.19\% & 34.41\% \\
Crazy Climber & 178.55\% & 364.24\% & 447.69\% \\
Defender &  & 53.25\% & 207.17\% \\
Demon Attack & 390.38\% & 424.65\% & 2151.65\% \\
Double Dunk & -350.00\% & 600.00\% & 981.25\% \\
Enduro & 67.81\% & 67.85\% & 157.96\% \\
Fishing Derby & 91.00\% & 89.66\% & 97.69\% \\
Freeway & 100.78\% & 102.75\% & 112.55\% \\
Frostbite & 2.20\% & 4.64\% & 33.40\% \\
Gopher & 120.42\% & 412.07\% & 727.95\% \\
Gravitar & -1.01\% & -2.63\% & -1.57\% \\
H.E.R.O. & 46.88\% & 56.73\% & 54.88\% \\
Ice Hockey & 57.84\% & 59.80\% & 70.59\% \\
James Bond & 94.03\% & 114.18\% & 161.04\% \\
Kangaroo & 98.37\% & 228.80\% & 420.77\% \\
Krull & 283.34\% & 520.11\% & 589.66\% \\
Kung-Fu Master & 56.49\% & 109.69\% & 145.99\% \\
Montezuma's Revenge & 0.60\% & 0.12\% & 0.41\% \\
Ms. Pacman & 3.73\% & 7.93\% & 6.88\% \\
Name This Game & 73.14\% & 121.30\% & 142.87\% \\
Phoenix &  & 166.25\% & 202.31\% \\
Pit Fall &  & -1.32\% & 2.55\% \\
Pong & 102.09\% & 106.57\% & 110.75\% \\
Private Eye & -0.57\% & -0.50\% & -1.95\% \\
Q*Bert & 37.03\% & 88.48\% & 91.06\% \\
River Raid & 25.21\% & 43.43\% & 74.31\% \\
Road Runner & 135.73\% & 654.15\% & 643.25\% \\
Robotank & 863.08\% & 763.08\% & 872.31\% \\
Seaquest & 6.41\% & 9.91\% & 35.52\% \\
Skiing &  & -121.69\% & 32.73\% \\
Solaris &  & 1.33\% & -13.77\% \\
Space Invaders & 98.81\% & 102.40\% & 190.76\% \\
Star Gunner & 378.03\% & 592.85\% & 653.02\% \\
Surround &  & 13.91\% & 76.82\% \\
Tennis & 129.93\% & 220.41\% & 92.52\% \\
Time Pilot & 99.58\% & 88.43\% & 140.30\% \\
Tutankham & 15.68\% & 40.53\% & 63.30\% \\
Up and Down & 28.34\% & 43.68\% & 200.02\% \\
Venture & 3.53\% & 5.58\% & 0.29\% \\
Video Pinball & -4.65\% & 2669.60\% & 7220.51\% \\
Wizard of Wor & -14.87\% & -17.30\% & 143.84\% \\
Yars Revenge &  & 8.68\% & 10.50\% \\
Zaxxon & 4.47\% & 92.86\% & 101.88\% \\
\end{tabular}
\caption{\label{table:at30_rnd_norm} Normalized scores for the human start condition (30 minutes emulator time).}
\end{table*}

\end{document}